\documentclass{article}

\usepackage[preprint]{neurips_2020}
\usepackage{amsmath}

\usepackage[utf8]{inputenc} 
\usepackage[T1]{fontenc}    
\usepackage{hyperref}       
\usepackage{url}            
\usepackage{booktabs}       
\usepackage{amsfonts}       
\usepackage{nicefrac}       
\usepackage{microtype}      
\usepackage{multirow}
\usepackage{multicol}
\usepackage{graphicx}
\usepackage{floatrow}
\usepackage{xargs} 
\usepackage{amsthm}
\newtheorem{lemma}{Lemma}
\newtheorem{theorem}{Theorem}
\newenvironment{rcases}
  {\left.\begin{aligned}}
  {\end{aligned}\right\rbrace}

\usepackage[colorinlistoftodos,prependcaption,textsize=tiny]{todonotes}
\newcommandx{\madian}[2][1=]{\todo[linecolor=red,backgroundcolor=red!25,bordercolor=red,#1]{#2}}
\newcommandx{\belinda}[2][1=]{\textcolor{red}{[Belinda:] #2}}
\newcommandx{\sinong}[2][1=]{\textcolor{red}{[Sinong:] #2}}
\usepackage{xcolor}

\title{Linformer: Self-Attention with Linear Complexity}

\author{%
Sinong Wang, Belinda Z. Li, Madian Khabsa, Han Fang, Hao Ma\\
Facebook AI, Seattle, WA\\
\texttt{\{sinongwang, belindali, hanfang, mkhabsa, haom\}@fb.com}\\
}

\begin{document}

\maketitle

\begin{abstract}
Large transformer models have shown extraordinary success in achieving state-of-the-art results in many natural language processing applications. However, training and deploying these models can be prohibitively costly for long sequences, as the standard self-attention mechanism of the Transformer uses $O(n^2)$ time and space with respect to sequence length. In this paper, we demonstrate that the self-attention mechanism can be approximated by a low-rank matrix. We further exploit this finding to propose a new self-attention mechanism, which reduces the overall self-attention complexity from $O(n^2)$ to $O(n)$ in both time and space. The resulting linear transformer, the \textit{Linformer}, performs on par with standard Transformer models, while being much more memory- and time-efficient. 
\end{abstract}

\section{Introduction}

Transformer models~\citep{vaswani2017attention} have become 
ubiquitous for wide variety of problems in natural language processing (NLP), including translation~\citep{ott2018scaling}, text classification, question answering, among others~\citep{raffel2019exploring,mohamed2019transformers}.
Over the last couple of years, 
the number of parameters in state-of-the-art NLP transformers 
has grown drastically, from the original 340 million introduced in BERT-Large to 
175 billion in GPT-3~\citep{brown2020language}. Although these large-scale models yield impressive results on wide variety of tasks,  
training and deploying such model are slow in practice. For example, the original BERT-Large model~\citep{devlin2019bert} takes four days to train on 16 Cloud TPUs, and the recent GPT-3~\citep{brown2020language} 
consumed orders of magnitude more petaflops / day to train compared to its predecessor, GPT-2~\citep{radford2019language}.
Beyond training, deploying Transformer models to real world applications is also expensive, 
usually requiring extensive distillation~\citep{hinton2015distilling} or compression.

The main efficiency bottleneck in Transformer models is its self-attention mechanism.
Here, each token's representation is updated by attending to \textit{all} other tokens in the previous layer.
This operation is key for retaining long-term information, giving Transformers the edge over recurrent models on long sequences.
However, attending to all tokens at each layer incurs a complexity of $O(n^2)$ with respect to sequence length. 
Thus, in this paper, we seek to answer the question: 
\textit{can Transformer models be optimized to avoid this quadratic operation, or is this operation required to maintain strong performance?}

Prior work has proposed several techniques for improving the efficiency of self-attention.
One popular technique is introducing sparsity into attention layers~\citep{child2019generating,qiu2019blockwise,beltagy2020longformer} by having each token attend to only a subset of tokens in the whole sequence. This reduces the overall complexity of the attention mechanism to $O(n\sqrt{n})$~\citep{child2019generating}. However, as shown in~\citet{qiu2019blockwise}, this approach suffers from a large performance drop with limited efficiency gains, i.e., a 2\% drop with only 20\% speed up. 
More recently, 
the Reformer~\citep{kitaev2019reformer} used 
locally-sensitive hashing (LSH) 
to reduce the self-attention complexity to $O(n\log(n))$.
However, in practice, 
the Reformer's efficiency gains only appear 
on sequences with length $> 2048$ (Figure 5 in~\cite{kitaev2019reformer}). 
Furthermore, the Reformer's multi-round hashing approach actually \textit{increases} the number of sequential operations, which 
further undermines their final efficiency gains. 

In this work, we introduce a novel approach for tackling the self-attention bottleneck in Transformers. Our approach is inspired by the key observation that \emph{self-attention is low rank}. More precisely, we show both theoretically and empirically that the stochastic matrix formed by self-attention can be approximated by a low-rank matrix. 
Empowered by this observation, we introduce a novel mechanism that reduces self-attention to an $O(n)$ operation in both space- and time-complexity: 
we decompose
the original scaled dot-product attention into multiple smaller attentions through linear projections, such that the combination of these operations forms a low-rank factorization of the original attention.
A summary of runtimes for various Transformer architectures, including ours, can be found in Table~\ref{tbl:summary}.

One predominant application of Transformers, that has seen the most gains, is using them as 
pretrained language models, whereby models are first pretrained with a language modeling objective on a large corpus, then finetuned on target tasks using supervised data~\citep{devlin2019bert,liu2019roberta,lewis2019bart}. 
Following~\cite{devlin2019bert}, we 
pretrain our model on BookCorpus~\citep{zhu2015aligning} plus English Wikipedia using masked-language-modeling objective. We observe similar pretraining performance 
to the standard Transformer model. We then finetune our pretrained models on three tasks from GLUE~\citep{DBLP:journals/corr/abs-1804-07461} and one sentiment analysis task, IMDB reviews~\citep{maas2011learning}. 
On these tasks, we find that our model performs comparably, or even slightly better, than the standard pretrained Transformer, while observing significant 
training and inference speedups.

\begin{table}[htbp]
    \caption{Per-layer time complexity and minimum number of sequential operations as a function of sequence length ($n$) for various architectures.}
\label{tbl:summary}
\begin{center}
\small
\begin{tabular}{@{}l c c}
\toprule
Model Architecture & Complexity per Layer & Sequential Operation \\\midrule
Recurrent & $O(n)$ & $O(n)$ \\
Transformer,~\citep{vaswani2017attention} & $O(n^2)$ & $O(1)$ \\
Sparse Tansformer,~\citep{child2019generating} & $O(n\sqrt{n})$ & $O(1)$ \\
Reformer,~\citep{kitaev2019reformer} & $O(n\log(n))$ & $O(\log(n))$ \\
Linformer & $O(n)$ & $O(1)$ \\
\bottomrule
\end{tabular}
\normalsize
\end{center}
\end{table}

\section{Backgrounds and Related works}

\subsection{Transformer and Self-Attention}

The Transformer is 
built upon the idea of Multi-Head Self-Attention (MHA), which allows the model to jointly attend to information at different positions from different representation subspaces. MHA is defined as 
\begin{equation}
\mbox{MultiHead}(Q, K, V) = \mbox{Concat}(\mbox{head}_1,\mbox{head}_2,\ldots,\mbox{head}_h)W^O,
\end{equation}
where $Q,K,V\in\mathbb{R}^{n\times d_m}$ are input embedding matrices, $n$ is sequence length, $d_m$ is the embedding dimension, and $h$ is the number of heads. Each head is defined as: 
\begin{equation}
\mbox{head}_i = \mbox{Attention}(QW_i^Q, KW_i^K, VW_i^V)
= \underbrace{\mbox{softmax}\left[\frac{QW_i^Q(KW_i^K)^T}{\sqrt{d_k}}\right]}_{P}VW_i^V,
\label{eq:selfattention}
\end{equation}
where 
$W_i^Q, W_i^K\in\mathbb{R}^{d_m\times d_k}, W_i^V\in\mathbb{R}^{d_m\times d_v}, W^O\in\mathbb{R}^{hd_v\times d_m}$ are learned matrices and $d_k, d_v$ are the hidden dimensions of the projection subspaces. For the rest of this paper, we will not differentiate between $d_k$ and $d_v$ and just use $d$. 

The self-attention defined in (\ref{eq:selfattention}) refers to a context mapping matrix $P\in\mathbb{R}^{n\times n}$. 
The Transformer uses $P$ to capture the input context for a given token, based on a combination of all tokens in the sequence.
However, computing
$P$ is expensive. It requires multiplying two $n\times d$ matrices, which is $O(n^2)$ in time and space complexity.
This quadratic dependency on the sequence length has become a bottleneck for Transformers.

\subsection{Related works}


There has been much prior literature on improving the efficiency of Transformers, especially the self-attention bottleneck. The most common techniques for model efficiency that can be applied to Transformers (some specific to Transformers, others more general-purpose) include:


\textbf{Mixed Precision}~\citep{micikevicius2017mixed}: 
Using half-precision or mixed-precision representations of floating points 
is popular in deep learning, 
and is also widely used in training Transformers~\citep{ott2019fairseq}. 
This technique can be further improved through Quantization Aware Training~\citep{jacob2018quantization,fan2020training}, where the weights are quantized during training and the gradients are approximated with the Straight-Through Estimator. This line of work is orthogonal to our approach, and we use mixed-precision training by default.

\textbf{Knowledge Distillation}~\citep{hinton2015distilling}: Knowledge distillation aims to transfer the ``knowledge" from a large teacher model to a lightweight student model. The student model is then used during inference. 
However this approach has drawbacks: It does not address speeding up the \textit{teacher} model during training, 
and moreover, student models usually suffer performance degradation compared to the teacher model. 
For example, when distilling a 12-layer BERT 
to a 6-layer 
BERT, 
the student model 
experiences an average 2.5\% performance drop on several benchmark tasks~\citep{sanh2019distilbert}.

\textbf{Sparse Attention}~\citep{child2019generating}: 
This technique improves the efficiency of self-attention by 
adding sparsity in the context mapping matrix $P$. For example, the Sparse Transformer~\citep{child2019generating} only computes $P_{ij}$ around the diagonal of matrix $P$ (instead of the all $P_{ij}$). Meanwhile, blockwise self-attention~\citep{qiu2019blockwise} divides 
$P$ into multiple blocks and only computes 
$P_{ij}$ within the selected blocks. However, these techniques also suffer 
a large performance degradation, while having only limited additional speed-up, i.e., 2\% drop with 20\% speed up. 

\textbf{LSH Attention}~\citep{kitaev2019reformer}: 
Locally-sensitive hashing (LSH) attention 
utilizes a 
multi-round hashing scheme 
when computing dot-product attention, 
which in theory 
reduces the self-attention complexity to $O(n\log(n))$. However, in practice, their complexity term has a large constant $128^2$ 
and it is 
only more efficient than the vanilla transformer when sequence length is extremely long.

\textbf{Improving Optimizer Efficiency}: 
Microbatching~\citep{huang2019gpipe} 
splits a batch into small microbatches (which can be fit into memory), and then separately runs forward and backward passes on them 
with gradient accumulation. 
Gradient checkpointing~\citep{chen2016training} saves memory by only caching activations of a subset of layers. The uncached activations 
are recomputed during backpropagation from the latest checkpoint. Both techniques trade off time for memory, and do not speed up inference.

As we've noted, most common techniques have limitations in reducing both the training and inference time/memory consumption, we investigate how to optimize the self-attention layers and introduce our approach next.

\section{Self-Attention is Low Rank\label{sec:analysis}}

In this section, we demonstrate that the self-attention mechanism, i.e., the context mapping matrix $P$, is low-rank. 

\begin{figure}
  \centering
  \includegraphics[width=5.5in]{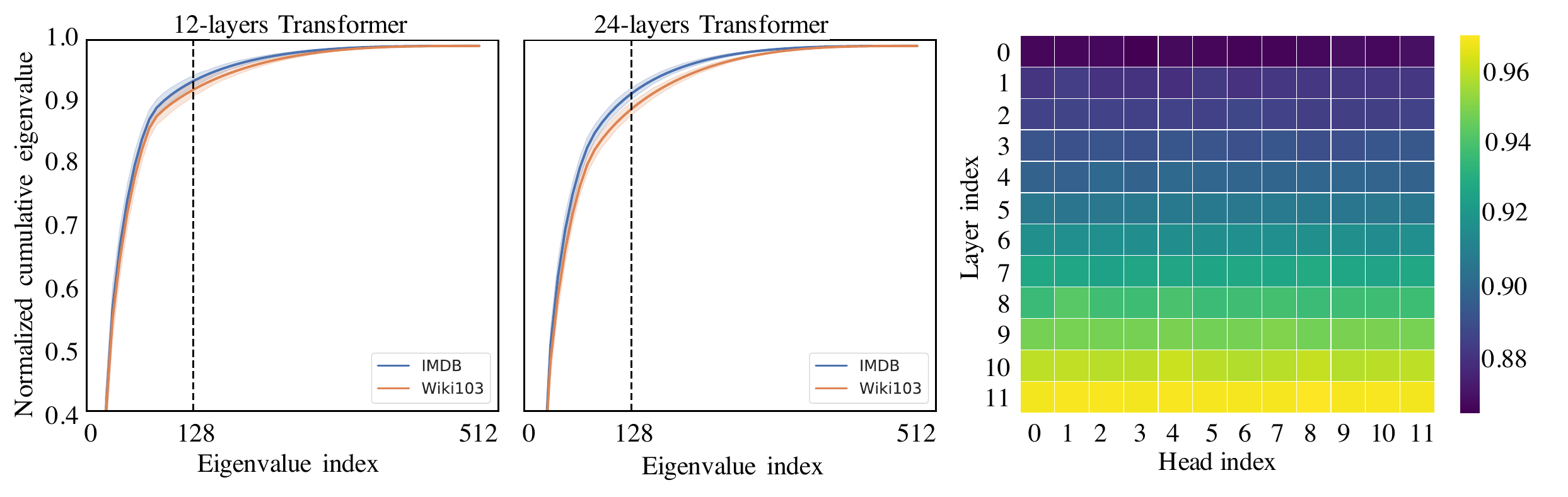}
  \caption{Left two figures are spectrum analysis of the self-attention matrix in pretrained transformer model~\citep{liu2019roberta} with $n=512$. The Y-axis is the normalized cumulative singular value of context mapping matrix $P$, and the X-axis the index of largest eigenvalue. The results are based on both RoBERTa-base and large model in two public datasets: Wiki103 and IMDB. The right figure plots the heatmap of normalized cumulative eigenvalue at the 128-th largest eigenvalue across different layers and heads in Wiki103 data.} 
  \label{fig:spectrum}
\end{figure}

We first provide a spectrum analysis of the context mapping matrix $P$. We use two pretrained transformer models, RoBERTa-base (12-layer stacked transformer) and RoBERTa-large (24-layer stacked transformer)~\citep{liu2019roberta} on two tasks: masked-language-modeling task on Wiki103~\citep{merity2016pointer} and classification task on IMDB~\citep{maas2011learning}. 
In Figure~\ref{fig:spectrum} (left), we apply singular value decomposition into $P$ across different layers and different heads of the model, and plot the normalized cumulative singular value averaged over 10k sentences.
The results exhibit a clear long-tail spectrum distribution across each layer, head and task. 
This implies that most of the information of matrix $P$ can be recovered from the first few largest singular values. 
In Figure~\ref{fig:spectrum} (right), we plot a heatmap of the normalized cumulative singular value at the 128-th largest singular value (out of 512). We observe that the spectrum distribution in 
higher layers is more skewed 
than in lower layers, 
meaning that, 
in higher layers, more information is concentrated in the largest singular values and the rank of $P$ is lower.

Below, we provide a theoretical analysis of the above spectrum results.
\begin{theorem} \emph{(self-attention is low rank)}
\label{thm:low-rank}
For any $Q,K,V\in\mathbb{R}^{n\times d}$ and $W^Q_i, W^K_i, W^V_i \in \mathbb{R}^{d\times d}$, for any column vector $w\in\mathbb{R}^{n}$ of matrix $VW^V_i$, there exists a low-rank matrix $\tilde{P}\in\mathbb{R}^{n\times n}$ such that
\begin{equation}
\Pr(\|\tilde{P}w^T-Pw^T\|<\epsilon\|Pw^T\|)>1-o(1)  \mbox{ and }  \text{rank}(\tilde{P})=\Theta(\log(n)),
\end{equation}
where the context mapping matrix $P$ is defined in (\ref{eq:selfattention}).
\end{theorem}
\begin{proof}
Based on the definition of the context mapping matrix $P$, we can write 
\begin{equation}
P = \mbox{ softmax}\underbrace{\left[\frac{QW_i^Q(KW_i^K)^T}{\sqrt{d}}\right]}_{A}=\exp{(A)}\cdot D_{A}^{-1},
\end{equation}
where $D_A$ is an $n\times n$ diagonal matrix. The main idea of this proof is based on the distributional Johnson–Lindenstrauss lemma~\citep{lindenstrauss1984extensions} (JL for short). We construct the approximate low rank matrix as $\tilde{P}=\exp{(A)}\cdot D_{A}^{-1}R^TR$, where $R\in\mathbb{R}^{k\times n}$ with i.i.d. entries from $N(0, 1/k)$.  We can then use the JL lemma to show that, for any column vector $w\in\mathbb{R}^{n}$ of matrix $VW_i^V$, when $k=5\log(n)/(\epsilon^2-\epsilon^3)$, we have 
\begin{equation}
\Pr\left(\|PR^TRw^T-Pw^T\|\leq \epsilon\|Pw^T\|\right) > 1-o(1).
\end{equation}
For more details, refer to the supplementary materials.
\end{proof}
Given the low-rank property of the context mapping matrix $P$, one straightforward idea is to use singular value decomposition (SVD) to approximate $P$ with a low-rank matrix $P_\text{low}$, as follows 
\begin{equation}\label{eq:svdapprox}
P \approx P_{\mbox{low}}=\sum\limits_{i=1}^{k}\sigma_iu_iv_i^T=\underbrace{\begin{bmatrix}
   \\
   u_1, \cdots,  u_k\\
   \\
\end{bmatrix}}_{k}\mbox{diag}\{\sigma_1, \cdots, \sigma_k\}\begin{rcases}
  \begin{bmatrix}
   &v_1 &\\
   &\vdots& \\
   &v_k &\\ 
\end{bmatrix}
\end{rcases}k
\end{equation}
where $\sigma_i$, $u_i$ and $v_i$ are the $i$ largest singular values and their corresponding singular vectors. Based on the results in Theorem~\ref{thm:low-rank} and the Eckart–Young–Mirsky Theorem~\citep{eckart1936approximation}, one can use
$P_\text{low}$ 
to approximate 
self-attention (\ref{eq:selfattention}) with $\epsilon$ error and $O(nk)$ time and space complexity. 
However, this approach requires performing an SVD decomposition in \textit{each} self-attention matrix, which adds additional complexity. 
Therefore, we propose another approach for 
low-rank approximation 
that avoids this added complexity.

\section{Model}

In this section, we propose a new self-attention mechanism which allows us to compute the contextual mapping $P\cdot VW_i^V$ in linear time and memory complexity with respect to sequence length.

The main idea of our proposed linear self-attention (Figure~\ref{fig:model}) is to add two linear projection matrices
$E_i, F_i\in \mathbb{R}^{n\times k}$ when computing key and value. 
We first project the original $(n\times d)$-dimensional key and value layers $KW_i^K$ and $VW_i^V$ into $(k \times d)$-dimensional projected key and value layers. We then compute an $(n\times k)$-dimensional context mapping matrix $\bar{P}$ using scaled dot-product attention.
\begin{align}
\overline{\mbox{head}_i} &= \mbox{Attention}(QW_i^Q, E_iKW_i^K, F_iVW_i^V)\notag\\
&=\underbrace{\mbox{softmax}\left(\frac{QW_i^Q(E_iKW_i^K)^T}{\sqrt{d_k}}\right)}_{\bar{P}: n\times k}\cdot\underbrace{F_iVW_i^V}_{k\times d},\label{eq:linearattenion}
\end{align}
Finally, we compute context embeddings for each head$_i$ using $\bar{P}\cdot(F_iVW_i^V)$. 
Note the above operations only require $O(nk)$ time and space complexity. 
Thus, if we can choose a very small projected dimension $k$, such that $k\ll n$, then we can significantly reduce the memory and space consumption.  The following theorem states that, when $k=O(d/\epsilon^2)$ (independent of $n$), one can approximate $P\cdot VW_i^V$ using linear self-attention (\ref{eq:linearattenion}) with $\epsilon$ error.  
\begin{figure*}
  \centering
  \includegraphics[width=5.5in]{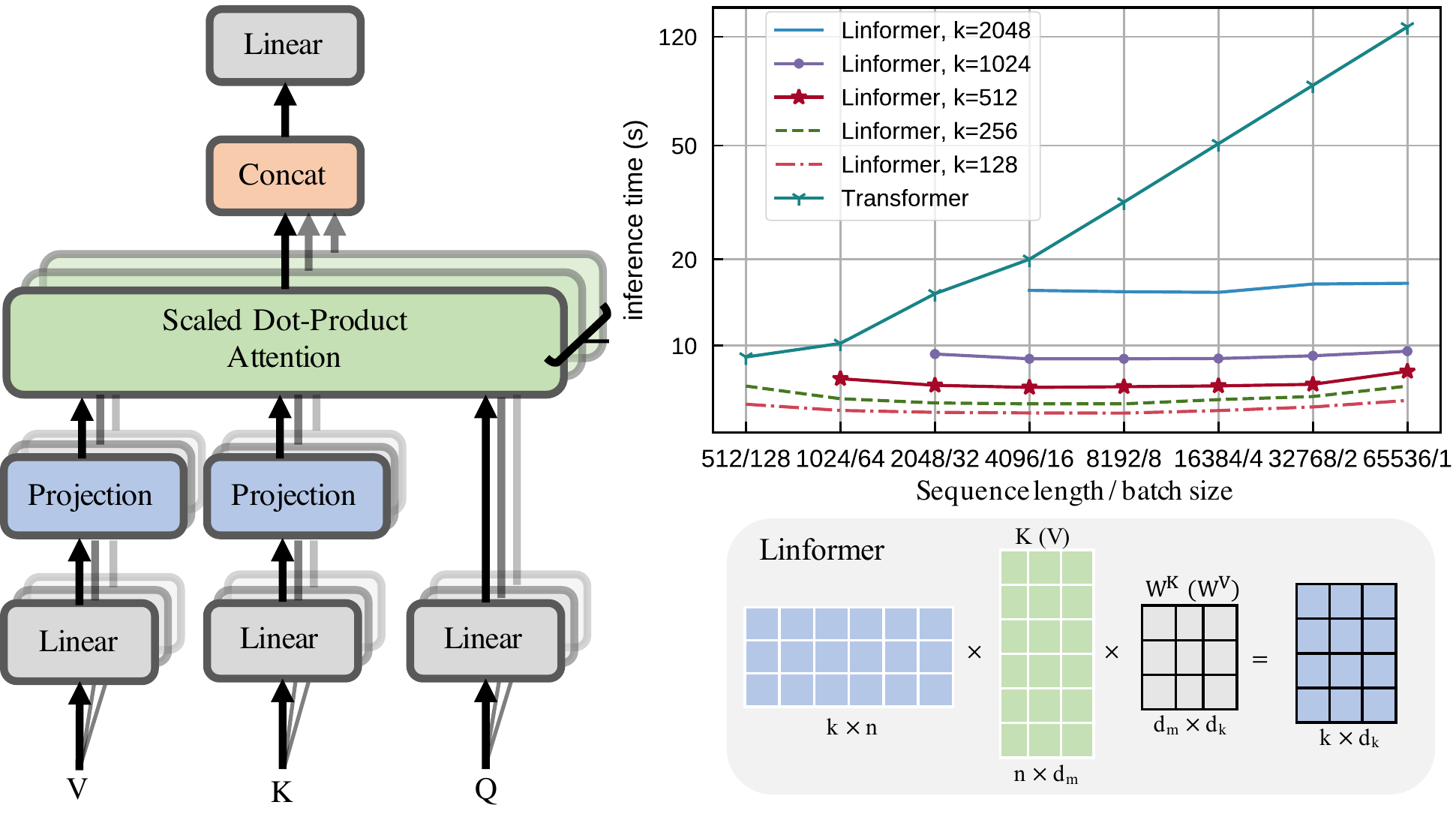}
  \caption{Left and bottom-right show architecture and example of our proposed multihead linear self-attention. Top right shows inference time vs. sequence length for various Linformer models.}
  \label{fig:model}
\end{figure*}
\begin{theorem} 
\label{thm:factorization} (\emph{Linear self-attention})
For any $Q_i, K_i, V_i\in\mathbb{R}^{n\times d}$ and $W_i^Q, W_i^K, W_i^V\in\mathbb{R}^{d\times d}$, if $k=\min\{\Theta(9d\log(d)/\epsilon^2), 5\Theta(\log(n)/\epsilon^2)\}$, then there exists matrices $E_i, F_i\in\mathbb{R}^{n\times k}$ such that, for any row vector $w$ of matrix $QW_i^Q(KW_i^K)^T/\sqrt{d}$, we have
\begin{equation}
\Pr\left(\|\mbox{\emph{softmax}}(wE_i^T)F_iVW_i^V-\mbox{\emph{softmax}}(w)VW_i^V\|\leq \epsilon\|\mbox{\emph{softmax}}(w)\|\|VW_i^V\|\right)>1-o(1)
\end{equation}
\end{theorem}
\begin{proof}
The main idea of proof is based on the distributional Johnson–Lindenstrauss lemma~\citep{lindenstrauss1984extensions}. We first prove that for any row vector $x\in\mathbb{R}^{n}$ of matrix $QW_i^Q(KW_i^K)^T/\sqrt{d_k}$ and column vector $y\in\mathbb{R}^{n}$ of matrix $VW_i^V$, 
\begin{align}
\Pr\left(\|\exp(xE_i^T)F_iy^T-\exp(x)y^T\|\leq \epsilon\|\exp(x)y^T\|\right)>1-2e^{-(\epsilon^2-\epsilon^3)k/4},\label{eq:main:exp_jl}
\end{align}
where $E_i=\delta R$ and $F_i=e^{-\delta}R$, where $R\in\mathbb{R}^{k\times n}$ with i.i.d. entries from $N(0, 1/k)$ and $\delta$ is a small constant. Applying the result in (\ref{eq:main:exp_jl}) to every row vector of matrix $A$ and every column vector of matrix $V$, one can directly prove that, for any row vector $A_i$ of matrix $A$,
\begin{align}\label{eq:submatrix_jl}
\Pr\left(\|\exp(A_iE_i^T)F_iV-\exp(A_i)V\|\leq \epsilon\|\exp(A_i)V\|\right)>1-o(1),
\end{align}
by setting $k=5\log(nd)/(\epsilon^2-\epsilon^3)$. This result does not utilize the low rank property of matrix $A$ (rank($A$)=$d$) and the resultant $k$ has a dependency on sequence length $n$. We will further utlize the fact that rank($A$)=$d$ to prove the choice of $k$ can be constant and independent of sequence length $n$. For more details, refer to the supplementary materials.
\end{proof}

In Figure~\ref{fig:model} (top right), we plot the inference speed of Linformer and standard Transformer versus sequence length, while holding the total number of tokens fixed. We see that while standard Transformer becomes slower at longer sequence lengths, the Linformer speed remains relatively flat and is significantly faster at long sequences.

\paragraph{Additional Efficiency Techniques} Several additional techniques can be introduced on top of Linformer to further optimize for both performance and efficiency:

\textbf{Parameter sharing between projections:} One can share parameters for the 
linear projection matrices $E_i, F_i$ across layers and heads. In particular, we experimented with 3 levels of sharing:  
\begin{itemize}
    \item Headwise sharing: for each layer, we share two projection matrices $E$ and $F$ such that $E_i= E$ and $F_i=F$ across all heads $i$.  
    \item Key-value sharing: we do headwise sharing, with the additional constraint of sharing the key and value projections. For each layer, we create a single projection matrix $E$ such that $E_i = F_i = E$ for each key-value projection matrix across all head $i$.
    \item Layerwise sharing: we use a single projection matrix $E$ across \textit{all layers}, for all heads, and for both key and value. 
\end{itemize}
For example, in a 12-layer, 12-head stacked Transformer model, headwise sharing, key-value sharing and layerwise sharing will introduce 24, 12, and 1 distinct linear projection matrices, respectively. 


\textbf{Nonuniform projected dimension:} One can choose a 
different projected dimension $k$ for different heads and layers. As shown in Figure~\ref{fig:spectrum} (right), the contextual mapping matrices in different heads and layers have distinct spectrum distributions, and heads in higher layer tend towards a more skewed distributed spectrum (lower rank). This implies one can choose a smaller projected dimension $k$ for higher layers. 

\textbf{General projections:} One can also choose different kinds of low-dimensional projection methods instead of a simple 
linear projection. For example, one can choose mean/max pooling, or convolution where the kernel and stride is set to $n/k$. The convolutional functions contain parameters that require training.

\section{Experiments}

In this section, we present experimental results for the 
the techniques described above. We analyze the techniques one-by-one and explore how they 
impact performance. 

\subsection{Pretraining Perplexities}

We first compare the pretraining performance of our proposed architecture against RoBERTa~\citep{liu2019roberta}, which is based on the Transformer. Following~\cite{devlin2019bert}, we use BookCorpus~\citep{zhu2015aligning} plus English Wikipedia as our pretraining set (3300M words).
All models are pretrained with the masked-language-modeling (MLM) objective, and the training for all experiments are parallelized across 64 Tesla V100 GPUs with 250k updates.

\textbf{Effect of projected dimension:} We 
experiment with various values for the projected dimension $k$. (We use the same $k$ across all layers and heads of Linformer.)
In the Figure~\ref{fig:perplexity}(a) and (b), we plot the validation perplexity curves for both the standard Transformer and the Linformer across different $k$, for maximum sequence lengths $n=512$ and $n=1024$. 
As expected, 
the Linformer performs better 
as projected dimension $k$ increases. 
However, even at $k=128$ for $n=512$ and $k=256$ for $n=1024$, Linformer's performance is already nearly on par with 
the original Transformer.
\begin{figure*}
  \centering
  \includegraphics[width=5.5in]{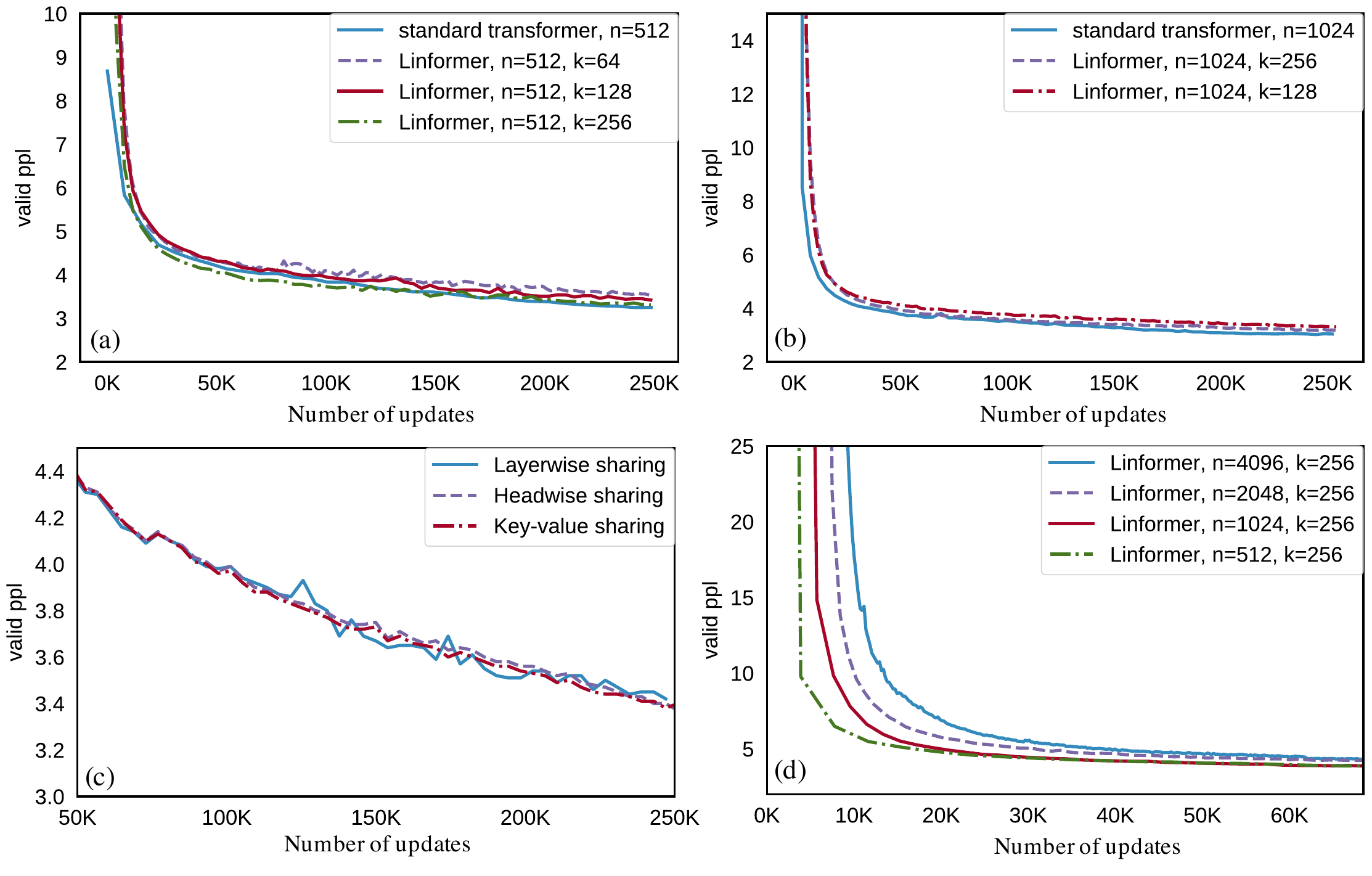}
  \caption{Pretraining validation perplexity versus number of updates.}
  \label{fig:perplexity}
\end{figure*}
\textbf{Effect of sharing projections:} In 
Figure~\ref{fig:perplexity}(c), we plot the validation perplexity curves for the three parameter sharing strategies (headwise, key-value, and layerwise) with $n=512$. 
Note that when we use just a single projection matrix (i.e. for layerwise sharing), the resulting Linformer model's validation perplexity almost matches that of the 
the non-shared model.
This suggests that we can decrease the number of additional parameters in our model, and consequently, it's memory consumption, without much detriment to performance.

\textbf{Effect of longer sequences:} We evaluate the effect of sequence length during Linformer pretraining. 
In the Figure~\ref{fig:perplexity}(d), we plot the validation perplexity for Linformer with $n\in\{512, 1024, 2048, 4096\}$, holding projected dimension $k$ fixed at $256$.
Note that as sequence length increases, even though our projected dimension is fixed, the final perplexities after convergence remain about the same.  
This further empirically supports our assertion that the Linformer is linear-time. 

\begin{table}[htbp]
\caption{Dev set results on benchmark natural language understanding tasks. 
The RoBERTa-base model here is pretrained with same corpus as BERT.}
\label{tbl:downstream_results}
\begin{center}
\small
\begin{tabular}{@{}l l  c  c  c  c  c  c @{}}
\toprule
$n$ & Model & SST-2 & IMDB & QNLI & QQP & Average \\\midrule
\multirow{9}{*}{512} &  
\cite{liu2019roberta}, RoBERTa-base
 & 93.1 & 94.1 & 90.9 & \textbf{90.9} & 92.25 \\
 & Linformer, 128 & 92.4 & 94.0 & 90.4 & 90.2 & 91.75 \\
 & Linformer, 128, shared kv & \textbf{93.4} & 93.4 & 90.3 & 90.3 & 91.85 \\
 & Linformer, 128, shared kv, layer & 93.2 & 93.8 & 90.1 & 90.2 & 91.83\\
 & Linformer, 256 & 93.2 & 94.0 & 90.6  & 90.5 & 92.08 \\
 & Linformer, 256, shared kv & 93.3 & 93.6 & 90.6 & 90.6 & 92.03 \\
 & Linformer, 256, shared kv, layer & 93.1  & 94.1 & \textbf{91.2} & 90.8 & \textbf{92.30} \\\midrule
\multirow{2}{*}{512} & \cite{devlin2019bert}, BERT-base & 92.7 & 93.5 & 91.8 & 89.6 & 91.90 \\
 & \cite{sanh2019distilbert}, Distilled BERT & 91.3 & 92.8 & 89.2 & 88.5 & 90.45\\\midrule
\multirow{3}{*}{1024} & Linformer, 256 & 93.0 & 93.8 & 90.4 & 90.4 & 91.90\\
 & Linformer, 256, shared kv & 93.0 & 93.6 & 90.3 & 90.4 & 91.83\\
 & Linformer, 256, shared kv, layer & 93.2 & \textbf{94.2} & 90.8 & 90.5 & 92.18 \\\bottomrule
\end{tabular}
\normalsize
\end{center}
\end{table}

\subsection{Downstream Results}

Thus far, we have only examined the pretraining perplexities of our model.
However, we wish to show that our conclusions hold after \textit{finetuning} on downstream tasks.
We finetune our Linformer on IMDB~\citep{maas2011learning} and SST-2~\citep{socher2013recursive} (sentiment classification), as well as QNLI (natural language inference)~\citep{rajpurkar2016squad}, and QQP (textual similarity)~\citep{chen2018quora}
We do the same with RoBERTa, 
12-layer BERT-base and 6-layer distilled BERT. All of our models, including the Transformer baselines, were pretrained with the same objective, pretraining corpus, and up to 250k updates (although our Linformer takes much less wall-clock time to get to 250k updates, and was consequently trained for less time). Results are listed in Table~\ref{tbl:downstream_results}.

We observe that the Linformer model ($n=512, k=128$) has comparable downstream performance to the RoBERTa model, and in fact even slightly outperforms it 
at $k=256$. 
Moreover, we note that although the Linformer's layerwise sharing strategy 
shares a single projection matrix across the entire model, 
it actually exhibits the best accuracy result of all three 
parameter sharing strategies.
Furthermore, the Linformer pretrained with longer sequence length $(n=1024, k=256)$ 
has similar results to the one pretrained with shorter length $(n=512, k=256)$,
this empirically supports the notion that \textit{the performance of Linformer model is mainly determined by the projected dimension $k$ instead of the ratio $n/k$.}

\subsection{Inference-time Efficiency Results}

In Table~\ref{tbl:inference_speed},
we report the inference efficiencies of Linformer (with layerwise sharing) against a standard Transformer. 
We benchmark both models' inference speed and memory on a 16GB Tesla V100 GPU card. 
We randomly generate data up to some sequence length $n$ and perform a full forward pass on a multiple batches. We also choose batch size based on the maximum batch size that can fit in memory, and our memory savings are computed based on this number.

\begin{table}[!htb]
    \begin{minipage}{.5\linewidth}
      \centering
        \begin{tabular}{@{~}r l l l l l@{~}} 
        \toprule
        \multirow{2}{*}{length $n$} & \multicolumn{5}{c}{projected dimensions $k$} \\
         & 128 & 256 & 512 & 1024 & 2048\\\midrule
        512 & 1.5x & 1.3x & - & - & -\\
        1024 & 1.7x & 1.6x  & 1.3x & - & - \\
        2048 & 2.6x  & 2.4x  & 2.1x & 1.3x & - \\
        4096 & 3.4x  & 3.2x  & 2.8x  & 2.2x  & 1.3x \\
        8192 & 5.5x & 5.0x& 4.4x & 3.5x  & 2.1x \\
        16384 & 8.6x & 7.8x& 7.0x & 5.6x & 3.3x\\
        32768 & 13x & 12x & 11x & 8.8x & 5.0x \\
        65536 & 20x & 18x & 16x & 14x& 7.9x \\
        \bottomrule
        \end{tabular}
    \end{minipage}%
    \begin{minipage}{.5\linewidth}
      \centering
        \begin{tabular}{@{~}r l l l l l@{~}} 
        \toprule
        \multirow{2}{*}{length $n$}& \multicolumn{5}{c}{projected dimensions $k$} \\
        & 128 & 256 & 512 & 1024 & 2048\\\midrule 
        512 & 1.7x & 1.5x & - & - & -\\
        1024 & 3.0x & 2.9x  & 1.8x & - & - \\
        2048 & 6.1x & 5.6x  & 3.6x & 2.0x & - \\
        4096 & 14x & 13x & 8.3x  & 4.3x  & 2.3x \\
        8192 & 28x & 26x & 17x & 8.5x & 4.5x \\
        16384 & 56x & 48x & 32x & 16x & 8x\\
        32768 & 56x & 48x & 36x & 18x & 16x \\
        65536 & 60x & 52x & 40x & 20x & 18x \\
        \bottomrule
        \end{tabular}
    \end{minipage} 
    \caption{Inference-time efficiency improvements of the Linformer over the Transformer, across various projected dimensions $k$ and sequence lengths $n$. 
    Left table shows time saved. 
    Right table shows memory saved.
    }
    \label{tbl:inference_speed}
\end{table}

From Table~\ref{tbl:inference_speed}, we see that even with 
$n = 512$ and 
$k = 128$, 
Linformer has 
$1.5\times$ faster inference time and 
allows for
a $1.7\times$ larger maximum batch size 
than the Transformer.
As sequence length increases, the inference-time speed-up and memory savings are even more dramatic.
We also plot inference times of both Linformer and Transformer on the 100 data samples in the top right of Figure~\ref{fig:model}.

\section{Conclusion}

Transformer models are notoriously slow to train and deploy 
in practice since their self-attention operations have $O(n^2)$ time and space complexity with respect to sequence length $n$. In this paper, we demonstrate, both theoretically and empirically, that the stochastic matrix formed by self-attention mechanism is low-rank. We further leverage this observation to propose a new, highly efficient self-attention mechanism. 
Through a combination of theoretical and empirical analysis, we demonstrate that our proposed approach is $O(n)$ 
with respect to sequence length.

\section*{Broader Impact}
Our work focuses on making Transformers more efficient by introducing a mechanism that reduces self-attention to linear-time complexity. Potential positive impacts of efficient transformers include increasing the accessibility of our models, both for deployment on devices, as well as during training for research purposes. It also has potential impact on training transformer on images since we can support very long sequences. 
Furthermore, there are positive environmental benefits associated with decreasing the power consumption of models.
As such, we see no immediate negative ethical or societal impacts of our work
beyond what applies to other core building blocks of deep learning.
 
\bibliography{references}
\bibliographystyle{iclr2020_conference}
\appendix
\newpage
\section{Proof of Theorem 1}
\begin{proof}
The main proof idea is based on the distributional Johnson–Lindenstrauss lemma~\citep{lindenstrauss1984extensions} (JL, for short), the following version is from~\citep{arriaga2006algorithmic}.
\begin{lemma}
Let $R$ be an $k\times n$ matrix, $1 \leq k \leq n$, with i.i.d. entries from $N(0, 1/k)$. For any $x, y \in\mathbb{R}^n$, we have 
\begin{align}
&\Pr\left(\|Rx\|\leq (1+\epsilon)\|x\|\right)>1-e^{-(\epsilon^2-\epsilon^3)k/4},\label{lm:jl_1}\\
&\Pr\left(\|xR^TRy^T-xy^T\|\leq \epsilon\|xy\|\right)>1-2e^{-(\epsilon^2-\epsilon^3)k/4}.\label{lm:jl_2}
\end{align}
\end{lemma}
For simplicity, we will omit the subscript $i$ for matrix $W_i^K$, $W_i^Q$, $W_i^V$, $E_i$ and $F_i$. We will regard $Q$ as $QW^Q$, $K$ as $KW^K$ and $V$ as $VW^V$. Define
\begin{equation}
A = \frac{QW_i^Q(KW_i^K)^T}{\sqrt{d}}
\end{equation}
Based on the definition of contextual mapping matrix $P$, we have
\begin{align}
P=&\mbox{ softmax}\left[\frac{QW_i^Q(KW_i^K)^T}{\sqrt{d}}\right]\notag\\
=&\exp{(A)}\cdot D_{A}^{-1},
\end{align}
where $D_A$ is an $n\times n$ diagonal matrix such that
\begin{equation}
(D_A)_{ii} = \sum\limits_{j=1}^n\exp{\left(A_{ji}\right)}
\end{equation}
Here we provide a constructive proof. Given any approximation error $\epsilon>0$, define the following matrix.
\begin{equation}
\tilde{P}=\exp{(A)}\cdot D_{A}^{-1}R^TR,
\end{equation}
where $R$ be an $k\times n$ matrix, $1 \leq k \leq n$, with i.i.d. entries from $N(0, 1/k)$. Clearly the rank of matrix $\tilde{P}$ satisifies
\begin{equation}
\mbox{rank}(\tilde{P})\leq \mbox{rank}(R) = k.
\end{equation}
We further show that, when $k=\log(n)$, we have that, for any column vector $w\in\mathbb{R}^n$,
\begin{equation}
\Pr\left(\|\tilde{P}h - Ph\|\leq \epsilon\|Ph\|\right)>1-o(1).
\end{equation}
This concludes the theorem. For any row vector $u\in\mathbb{R}^n$ of matrix $P$ and any column vector $w\in\mathbb{R}^n$ of matrix $VW^V$, applying the JL Lemma, we can obtain
\begin{equation}
\Pr\left(\|uR^tRw^T - uw^T\|\leq \epsilon\|uw^T\|\right)>1-2e^{-(\epsilon^2-\epsilon^3)k/4}.
\end{equation}
Therefore, we have
\begin{align}
\Pr\left(\|\tilde{P}w^T - Pw^T\|\leq \epsilon\|Pw^T\|\right)=&\Pr\left(\|PR^TRw^T - Pw^T\|\leq \epsilon\|Pw^T\|\right)\notag\\
\overset{(a)}{\geq}&1 - \sum\limits_{x\in P}\Pr\left(\|xR^TRw^T - xw^T\| > \epsilon\|xw^T\|\right)\notag\\
\overset{(b)}{>}&1-2ne^{-(\epsilon^2-\epsilon^3)k/4}.
\end{align}
The above, step (a) is based on the union bound. The step (b) is utilizing the result of JL Lemma. Let $k=5\log(n)/(\epsilon^2-\epsilon^3)$, then theorem follows.
\end{proof}

\section{Proof of Theorem 2}

\begin{proof}
Define $E=\delta R$ and $F=e^{-\delta}R$, where $R\in\mathbb{R}^{n\times k}$ with i.i.d. entries from $N(0, 1/k)$, $\delta$ is a constant with $\delta=1/2^n$. We will first prove that for any row vector $x\in\mathbb{R}^{n}$ of matrix $QK^T$ and column vector $y\in\mathbb{R}^{n}$ of matrix $V$, 
\begin{align}
\Pr\left(\|\exp(xE^T)Fy^T-\exp(x)y^T\|\leq \epsilon\|\exp(x)y^T\|\right)>1-2e^{-(\epsilon^2-\epsilon^3)k/4}.\label{eq:exp_jl}
\end{align}
Based on the triangle inequality, we have
\begin{align}
\|\exp(xE^T)Fy\exp(x)y^T\|&\leq \|\exp(xE^T)Fy-\exp(x)R^TRy\| + \|\exp(x)R^TRy-\exp(x)y^T\|\notag\\
&\overset{(a)}{\leq} (1+\epsilon)\|y\|\|\exp(xE^T)-\exp(x)R^T\| + \|\exp(x)R^TRy-\exp(x)y^T\| \notag\\
&\overset{(b)}{\leq} \|\exp(x)R^TRy-\exp(x)y^T\| + o(\|\exp(x)\|\|y\|)\notag\\
&\overset{(c)}{\leq} \epsilon\|\exp(x)\|\|y\| + o(\|\exp(x)\|\|y\|)
\end{align}
The above, step (a) is based on the Cauchy inequality and JL Lemma in (\ref{lm:jl_1}). The step (b) utilizes the fact that exponential function is Lipchitz continuous in a compact region. Then we can choose a small enough $\delta$, i.e., $\delta=\theta(1/n)$ such that
\begin{equation}
    \|\exp(\delta xR) - \exp(\delta x)R\|=o(\|\exp(x)\|)
\end{equation}
The step (c) is based on the JL Lemma defined in (\ref{lm:jl_2}).

Applying the result in (\ref{eq:exp_jl}) to every row vector of matrix $A$ and every column vector of matrix $V$, one can directly prove that, for any row vector $A_i$ of matrix $A$,
\begin{align}\label{eq:submatrix_jl}
\Pr\left(\|\exp(A_iE^T)FV-\exp(A_i)V\|\leq \epsilon\|\exp(A_i)\|\|V\|\right)>1-o(1),
\end{align}
by setting $k=5\log(nd)/(\epsilon^2-\epsilon^3)$. This result does not utilize the low rank property of matrix $A$ (rank($A$)=$d$) and the resultant $k$ has a dependency on sequence length $n$. We will further prove the choice of $k$ can be constant and independent of sequence length $n$. 

Based on the fact that rank($A$)=$d$, we can find a row submatrix $A_s\in\mathbb{R}^{2d\times d}$ of matrix $\exp(AE^T)FH$ such that rank($A_s$)=$d$. Applying the result in (\ref{eq:exp_jl}) to every row vector of matrix $A_s$ and every column vector of matrix $V$, and $k=9\log(d)/(\epsilon^2-\epsilon^3)$, we can obtain that, for any row vector $A_i^s$ of matrix $A^s$, 
\begin{align}
\Pr\left(\|\exp(A_i^sE^T)FV-\exp(A_i^s)V\|\leq \epsilon\|\exp(A_i^s)\|\|V\|\right)>1-o(1),
\end{align}
Furthermore, define the matrix $\Gamma\in\mathbb{R}^{n\times 2d}$ as
\begin{equation}
\Gamma=\begin{bmatrix}
\exp(AE^T)FV\\
\exp(A)V
\end{bmatrix}\cdot\begin{bmatrix}
\exp(A_sE^T)FV\\
\exp(A_s)V
\end{bmatrix}^{-1}
\end{equation}
We have that, for any row vector $A_i$ of matrix $A$, $1\leq i \leq n$.
\begin{align}
\|\exp(A_iE^T)FV-\exp(A_i)V\|=&\|\Gamma_i\exp(A^sE^T)FV-\Gamma_i\exp(A^s)V\|\notag\\
\overset{(a)}{\leq}&\left\|[\exp(A^sE^T)FV-\exp(A^s)V]^T\right\|_2\|\Gamma_i\|\notag\\
\overset{(b)}{\leq}&\Theta(d)\|\exp(A^sE^T)FV-\exp(A^s)V\|_F\notag\\
=&\Theta(d) \sum\limits_{i=1}^{2d}\|\exp(A_i^sE^T)FV-\exp(A_i^s)V\|\notag\\
\overset{(c)}{\leq} & \epsilon\Theta(d)\sum\limits_{i=1}^{2d}\|\exp(A_i^s)\|\|V\|\notag\\
\leq &\epsilon\Theta(d)\|\exp(A^s)\|\|V\|\notag
\end{align}
The above, step (a) utilizes the inequality $\|Ax\|\leq\|A\|_2\cdot\|x\|$, where $\|A\|_2=\sqrt{\lambda_{\max}(A^TA})$ ($\lambda_{\max}(\cdot)$ is the largest eigenvalue) is the spectrum norm of a matrix $A$. The step (b) is based on matrix norm inequality $\|A\|_2\leq\|A\|_F$, where $\|A\|_F=(\sum_{1\leq i,j\leq n}A_{ij}^2)^{1/2}$ is the Frobenius norm of matrix $A$. The step (c) is based on the results of (\ref{eq:submatrix_jl}).
\end{proof}
\end{document}